\documentclass{article}
\pdfoutput=1
\usepackage{ijcai21}

\usepackage{times}
\usepackage{soul}
\usepackage{url}
\usepackage[hidelinks]{hyperref}
\usepackage[utf8]{inputenc}
\usepackage[small]{caption}
\usepackage{graphicx}
\usepackage{amsmath}
\usepackage{booktabs}
\usepackage{microtype}

\usepackage{amsthm}
\usepackage{amsfonts}
\usepackage{algorithm}
\usepackage{algorithmic}
\usepackage{subcaption}
\newtheorem{lemma}{Lemma}
\newtheorem{theorem}{Theorem}

\title{Robust Regularization with Adversarial Labelling of Perturbed Samples}
\author{
Xiaohui Guo\textsuperscript{\rm 1}
\and
Richong Zhang\textsuperscript{\rm 2}\footnote{Corresponding author:\texttt{zhangrc@act.buaa.edu.cn}}
\and
Yaowei Zheng\textsuperscript{\rm 2}
\And
Yongyi Mao\textsuperscript{\rm 3}
\affiliations
\textsuperscript{\rm 1}Hangzhou Innovation Institute, Beihang University, Hangzhou, China\\
\textsuperscript{\rm 2}BDBC and SKLSDE, School of Computer Science and Engineering, Beihang University, Beijing, China\\
\textsuperscript{\rm 3}School of Electrical Engineering and Computer Science, University of Ottawa, Ottawa, Canada\\
\emails
\{guoxh, zhangrc\}@act.buaa.edu.cn,
hiyouga@buaa.edu.cn,
ymao@uottawa.ca
}

\begin{document}

\maketitle

\begin{abstract}

Recent researches have suggested that the predictive accuracy of neural network may contend with its adversarial robustness. This presents challenges in designing effective regularization schemes that also provide strong adversarial robustness. Revisiting Vicinal Risk Minimization (VRM) as a unifying regularization principle, we propose Adversarial Labelling of Perturbed Samples (ALPS) as a regularization scheme that aims at improving the generalization ability and adversarial robustness of the trained model. ALPS trains neural networks with synthetic samples formed by perturbing each authentic input sample towards another one along with an adversarially assigned label. The ALPS regularization objective is formulated as a min-max problem, in which the outer problem is minimizing an upper-bound of the VRM loss, and the inner problem is L$_1$-ball constrained adversarial labelling on perturbed sample. The analytic solution to the induced inner maximization problem is elegantly derived, which enables computational efficiency. Experiments on the SVHN, CIFAR-10, CIFAR-100 and Tiny-ImageNet datasets show that the ALPS has a state-of-the-art regularization performance while also serving as an effective adversarial training scheme.

\end{abstract}

\section{Introduction}

Despite the stunning power of deep learning, regularization remains as an important technique in deep neural networks. In general, regularization may broadly refer to any technique that prevents a model from overfitting to the training data but generalizing poorly on the unseen data. Traditional regularization techniques are ``data-independent'', in the sense that they do not explicitly exploit the training data. These schemes, including, for example, weight decay \cite{krogh1992simple} and Dropout \cite{srivastava2014dropout}, impose additional constraints on the model so as to restrict its capacity. Label smoothing \cite{szegedy2016rethinking} may also be regarded as a data-independent regularization scheme, where the training labels are ``softened'', leading to smoother classification boundaries. Effective as they are, the ignorance of the data structure with these data-independent approaches is arguably a major limitation.

Indeed there exists another class of regularization schemes, which are ``data-dependent''. The most well-known example in this class is perhaps data augmentation, where additional examples are constructed from the training set based on certain prior knowledge. Each of the new examples then provides an additional constraint, thereby regularizing the model. To date, properly designed data augmentation schemes remain as the most powerful regularization techniques \cite{cubuk2019autoaugment}.

A more recent data-dependent regularization technique is known as MixUp \cite{zhang2018mixup}. It creates synthetic examples by interpolating the training examples and assigns a soft-label heuristically to each synthetic example by a similar interpolation of the labels. This approach is shown to be very effective in \cite{zhang2018mixup}, where the authors attribute this effectiveness to the Vicinal Risk Minimization (VRM) principle \cite{chapelle2000vicinal}, dated back 20 years ago. Briefly VRM insists training a model not using the training set {\em per se}, but using samples drawn from a distribution, called a ``vicinal distribution'', that ``smears'' the training data to their vicinity. Indeed, using a particular choice of such distribution, the original VRM approach results in several effective regularization schemes \cite{chapelle2000vicinal}. 

Recently adversarial training techniques \cite{goodfellow2014explaining} are proposed to improve the robustness of a model against the adversarial attacks. In these schemes, a synthetic example is produced by slightly perturbing a training example in a peculiar direction and the label of the original training example is assigned to the synthetic example. Such synthetic data are then used to train the model. Although these techniques are originally proposed to improve the adversarial robustness of a model, they also can be used as regularization techniques (see, {\em e.g.}, \cite{miyato2018virtual}. This is sensible since adversarial vulnerability and over-fitting share some similarities: both are caused by the close proximity of training examples to the classification boundaries. However recent research has revealed that the objective adversarial training and that of regularization, although correlated, are not completely aligned \cite{madry2018towards,tsipras2018robustness,zhang2019theoretically}: excessively optimizing the model towards adversarial robustness may in fact hurt its prediction accuracy. On the other hand, as observed in \cite{verma2019manifold}, effectively regularized models, such as those by MixUp, are not sufficiently robust against adversarial attacks.

At present, it remains unclear how much adversarial robustness a model has to trade in so as to generalize well. Nonetheless, the interest of this work is to develop a regularization scheme that not only allows the model to generalize well but also make it as robust as possible.  We call such an objective ``robust regularization''.

This work starts with a revisit of the VRM framework, since moving a data point to its vicinity appears to be an essential ingredient both for regularization and for adversarial training. In this revisit, we discover that label smoothing may also be unified under the VRM framework. More interestingly, view as a VRM technique, MixUp appears to have combined the advantages of both label smoothing and the original VRM in the following sense: the vicinal distribution used in MixUp is over the joint space of both inputs and labels whereas the vicinal distributions used in label smoothing and the original VRM are only over one of the two spaces. This inspires us to set up a vicinal distribution in a way similar to MixUp, where vicinity is defined in the joint input-label space. We then arrive at a new VRM loss function. We derive a regularization scheme based on minimizing an upper bound of this loss. This formulation, also adopting a min-max optimization as in adversarial training, has an elegant property, namely, that the inner maximization can be solved in a closed form. This allows us to arrive at a very simple training algorithm, which we designate ``adversarial labelling of perturbed samples'', or ALPS. Briefly, ALPS moves a training example $\boldsymbol{x}$ in the direction towards another example $\boldsymbol{x}'$ and assigns an ``adversarial label'' to the moved example; here the adversarial label is the soft-label obtained by perturbing the label of $\boldsymbol{x}$ to the least likely label predicted by the model. The combination of sample interpolation and adversarial labelling in ALPS then allow the scheme to be effective both as a regularization scheme and as an adversarial training scheme.

Experiments on the SVHN, CIFAR-10, CIFAR-100 and Tiny-ImageNet datasets confirm that ALPS is among the best regularization schemes. At the same time, it also offers great adversarial robustness to the trained model, much more than existing regularization schemes.

\section{Related Work}

Data augmentation includes a family of regularization strategy, from geometrically translation/crop transformation, perturbing samples with (adversarial) noise, to more advanced schemes such as AutoAugment \cite{cubuk2019autoaugment}, which automatically searching for augmentation policies from data.

MixUp \cite{zhang2018mixup} linearly interpolates data examples to augment training data, and encourages linear behaviour in the learned model. Manifold MixUp \cite{verma2019manifold} extends this approach to the hidden layers in the neural network. CutMix \cite{yun2019cutmix} cuts and pastes image patches to construct new images, and mix the labels in proportion to the patch sizes. To prevent an underfitting phenomenon ``manifold intrusion''  in Mixup, both AdaMixup \cite{guo2019mixup} and AugMix \cite{hendrycksaugmix} introduce elaborately devised penalty loss functions to learn better mixing policies.

Adversarial training \cite{goodfellow2014explaining} starts from improving the model's robustness against adversarial attacks by perturbing training samples within a bounded-norm ball, where regularization arise as a knock-on effect. Layer-wise adversarial training \cite{sankaranarayanan2018regularizing} perturbs the activation of intermediate layers. Adversarial labelling \cite{kurakin2017adversarial} uses the least-likely label as a strong adversary. To tackle the over-confidence of predictions, the label smoothing regularization scheme spreads the probability of the ground-truth label slightly to other labels. This approach appears to also provide some adversarial robustness in the trained model \cite{shafahi2019batch}.

Data-independent regularization methods fall in the earlier paradigm, which include, for example, early stopping, weight decay, Dropout \cite{srivastava2014dropout}. Similar to Dropout, DropBlock \cite{ghiasi2018dropblock} generalizes input Cutout \cite{yun2019cutmix} by applying Cutout at every feature map in convolutions networks.

\section{Vicinal Risk Minimization as a General Principle for Regularization}

Consider training a neural network classifier $f:\mathcal{X}\to\mathcal{Y}$ for a $K$-class classification problem, where $\mathcal{X}$ is the input space and $\mathcal{Y}$ is the label space. Note that here we take $\mathcal{Y}$ as space of all distributions over the set of all class labels $\{1,2,\ldots,K\}$, namely, each element $\boldsymbol{y}\in\mathcal{Y}$ is a $K$-dimensional vector representing a label distribution. Let $\mathcal{D}:=\{(\boldsymbol{x}_i,\boldsymbol{y}_i)\}_{i=1}^n$ denote the training set, where each $\boldsymbol{x}_i\in\mathcal{X}$ is a training example, and each $\boldsymbol{y}_i\in\mathcal{Y}$ is the label of $\boldsymbol{x}_i$ represented as a one-hot vector.

Let $\ell:\mathcal{Y}\times\mathcal{Y}$ denote the usual cross-entropy loss. Under the Empirical Risk Minimization (ERM) formulation \cite{Vapnik1998}, the training loss is defined as the empirical risk
\begin{equation}
\mathcal{L}=\frac{1}{n}\sum_{i=1}^n\ell(f(\boldsymbol{x}_i),\boldsymbol{y}_i)
\end{equation}

One can view the above empirical risk as the expectation of the cross-entropy loss over the {\em empirical distribution} given by $\mathcal{D}$. More precisely, let the empirical distribution $\widetilde{p}$ on $\mathcal{X}\times\mathcal{Y}$ be defined as
\begin{equation}
\widetilde{p}(\boldsymbol{x},\boldsymbol{y}):=\frac{1}{n}\sum_{i=1}^n\delta_{x_i}(\boldsymbol{x})\delta_{y_i}(\boldsymbol{y})
\end{equation}
where $\delta_{x_i}(\cdot)$ and $\delta_{y_i}(\cdot)$ are the Dirac measures that put probability 1 to $\boldsymbol{x}_i$ and to $\boldsymbol{y}_i$. It then follows
\begin{equation}\label{eq:ERM}
\mathcal{L}=\int\ell(f(\boldsymbol{x}),\boldsymbol{y})\widetilde{p}(\boldsymbol{x},\boldsymbol{y}){\rm d}\boldsymbol{x}{\rm d}\boldsymbol{y}
\end{equation}

It is known that ERM may suffer from overfitting, particularly when the sample size is small or when the model capacity is large. When this occurs, although the loss value $\ell(f(\boldsymbol{x}),\boldsymbol{y})$ is low on each training point $(\boldsymbol{x},\boldsymbol{y})\in\mathcal{D}$, the loss value $\ell(f(\boldsymbol{x}),\boldsymbol{y})$ may still be high on the unseen data point $(\boldsymbol{x},\boldsymbol{y})$.

In \cite{chapelle2000vicinal}, a Vicinal Risk Minimization (VRM) framework is proposed, where the empirical distribution $\widetilde{p}$ in (\ref{eq:ERM}) is replaced with a ``vicinal distribution'' $\widehat{p}$, giving rise to the training objective
\begin{equation}\label{eq:VRM}
\widehat{\mathcal{L}}=\int\ell(f(\boldsymbol{x}),\boldsymbol{y})\widehat{p}(\boldsymbol{x},\boldsymbol{y}){\rm d}\boldsymbol{x}{\rm d}\boldsymbol{y}
\end{equation}

Specifically, the authors of \cite{chapelle2000vicinal} advocate a particular form of the vicinal distribution, $\widehat{p}_{\rm CWBV}$ (``CWBV'' are initials of the authors of \cite{chapelle2000vicinal}), using a spherical Gaussian kernel function $\mathcal{N}(\boldsymbol{x}-\boldsymbol{x}_i,\sigma^2\mathbf{I})$ in place of each Dirac measure $\delta_{x_i}$:
\begin{equation}
\widehat{p}_\mathrm{CWBV}(\boldsymbol{x},\boldsymbol{y}):=\frac{1}{n}\sum_{i=1}^n\mathcal{N}(\boldsymbol{x}-\boldsymbol{x}_i;\sigma^2\mathbf{I})\delta_{y_i}(\boldsymbol{y})
\end{equation}

This approach corresponds to adding a spherical Gaussian noise to the training example $\boldsymbol{x}_i$ and, as shown in \cite{chapelle2000vicinal} that it gives rise to various regularization schemes, such as Ridge Regression, Constrained logistic classifier and Tangent-Prop. 

In this paper, we argue that the effectiveness of VRM can be explained by the fact that the vicinal distribution $\widehat{p}$ covers regions of $\mathcal{X}\times\mathcal{Y}$ beyond the support of the empirical distribution $\widetilde{p}$ ({\em i.e.}, the training points). This presents opportunities for the learned model to generalize better on unseen data. Indeed, beyond the original proposal of \cite{chapelle2000vicinal}, more recent regularization schemes may also be viewed as VRM with a special choice of vicinal distribution.

\paragraph{Label Smoothing as VRM}
We will use $\vec{\boldsymbol{1}},\vec{\boldsymbol{2}},\ldots,\vec{\boldsymbol{K}}$ to denote respectively the one-hot vector representations of labels $1,2,\ldots,K$. For any label $k\in\{1,2,\ldots,K\}$, we define the smoothed soft-label distribution $\pi_{k}$ on $\mathcal{Y}$ as
\begin{equation}
\pi_{\vec{k}}(\boldsymbol{y}):=(1-\epsilon)\delta_{\vec{k}}(\boldsymbol{y})+\frac{\epsilon}{K-1}\sum_{j\neq k}\delta_{\vec{j}}(\boldsymbol{y})
\end{equation}

Now define the vicinal distribution $\widehat{p}_{\rm LS}$ on $\mathcal{X}\times\mathcal{Y}$ by
\begin{equation}
\widehat{p}_{\mathrm{LS}}(\boldsymbol{x},\boldsymbol{y}):=\frac{1}{n}\sum_{i=1}^n\delta_{x_i}(\boldsymbol{x})\pi_{y_i}(\boldsymbol{y})
\end{equation}

Under this vicinal distribution, the minimization objective in VRM for label smoothing becomes
\begin{align}
\widehat{\mathcal{L}}_{\mathrm{LS}}
&=\int\ell(f(\boldsymbol{x}),\boldsymbol{y})\widehat{p}_\mathrm{LS}(\boldsymbol{x},\boldsymbol{y}){\rm d}\boldsymbol{x}{\rm d}\boldsymbol{y}\nonumber\\
&=\frac{1}{n}\sum_{i=1}^n\!
     \left[\!(1\!-\!\epsilon)\ell(f(\boldsymbol{x}_i),\boldsymbol{y}_i)\!+\! 
     \frac{\epsilon}{K\!-\!1}\sum_{\vec{\boldsymbol{j}}\neq\boldsymbol{y}_i}\!
     \ell\!\left(f(\boldsymbol{x}_i),\vec{\boldsymbol{j}}\right)\!\right]\nonumber\\
&=\frac{1}{n}\sum\limits_{i=1}^n\ell\left(f(\boldsymbol{x}_i),\widehat{\boldsymbol{y}}_i\right)\label{eq:LS}
\end{align}
where $\widehat{\boldsymbol{y}}_i$ is the $K$-dimensional probability vector that puts probability $1-\epsilon$ to the label represented by $\boldsymbol{y}_i$ and puts probability $\frac{\epsilon}{K-1}$ to each of the remaining labels. We note that the final equality above follows from the fact that the cross-entropy loss $\ell$ is linear in its second argument. It can be verified that the expression in (\ref{eq:LS}) is precisely the loss function minimized in label smoothing \cite{szegedy2016rethinking}. Thus we have proved the following lemma.
\begin{lemma}
Label smoothing is a special case of VRM, the minimization objective of which is specified in the form of (\ref{eq:LS}).
\end{lemma}
To the best of our knowledge, this connection between label smoothing and VRM as suggested in the lemma is reported for the first time. 

\paragraph{MixUp as VRM}
MixUp may also be viewed as VRM with a particular choice of vicinal distribution. This has been shown in \cite{zhang2018mixup}, which we now recapitulate. 

For the convenience of notation, for any two vectors $\boldsymbol{s}$ and $\boldsymbol{s}'$ having the same length and any scalar $\lambda\in[0,1]$, we denote
\begin{equation}
\mathbb{M}(\boldsymbol{s},\boldsymbol{s}';\lambda):=\lambda\boldsymbol{s}+(1-\lambda)\boldsymbol{s}'
\end{equation}

For any given pair $(\boldsymbol{x},\boldsymbol{y})\in\mathcal{X}\times\mathcal{Y}$, if we draw a random variable pair $(\boldsymbol{X}',\boldsymbol{Y}')$ from $\widetilde{p}$ and draw a random variable $\Lambda$ from a beta distribution $\text{Beta}(\alpha,\alpha)$ for some prescribed $\alpha$, then $\mathbb{M}(\boldsymbol{x},\boldsymbol{X}';\Lambda)$ and $\mathbb{M}(\boldsymbol{y},\boldsymbol{Y}';\Lambda)$ are a pair of random variables, and we will use $q(\cdot|\boldsymbol{x},\boldsymbol{y})$ to denote their joint distribution. Note that for any given $(\boldsymbol{x},\boldsymbol{y})\in\mathcal{X}\times\mathcal{Y}$, $q(\cdot|\boldsymbol{x},\boldsymbol{y})$ is a distribution over $\mathcal{X}\times\mathcal{Y}$. Now we define a vicinal distribution $\widehat{p}_{\rm MixUp}$ as follows: for any $(\boldsymbol{x},\boldsymbol{y})\in\mathcal{X}\times\mathcal{Y}$
\begin{equation}
\widehat{p}_\mathrm{MixUp}(\boldsymbol{x},\boldsymbol{y}):=\int q(\boldsymbol{x},\boldsymbol{y}|\boldsymbol{x}',\boldsymbol{y}')\widetilde{p}(\boldsymbol{x}',\boldsymbol{y}'){\rm d}\boldsymbol{x}'{\rm d}\boldsymbol{y}'
\end{equation}

Then under VRM, the minimization objective is 
\begin{align}
 &\widehat{\mathcal{L}}_\mathrm{MixUp}\nonumber\\
=&\int\ell(f(\boldsymbol{x}),\boldsymbol{y})\widehat{p}_{\mathrm{MixUp}}(\boldsymbol{x},\boldsymbol{y}){\rm d}\boldsymbol{x}{\rm d}\boldsymbol{y}\nonumber\\
=&\mathbb{E}_{\lambda\sim\text{Beta}(\alpha,\alpha)}\frac{1}{n^2} \sum_{i=1}^n\!\sum_{j=1}^n\!
\ell\!\left(f\!\left(\mathbb{M}(\boldsymbol{x}_i,\boldsymbol{x}_j;\lambda)\right),\mathbb{M}(\boldsymbol{y}_i,\boldsymbol{y}_j;\lambda)\right)
\end{align}

Now it is easy to see that MixUp essentially takes a stochastic approximation of loss $\widehat{\mathcal{L}}_{\mathrm{MixUp}}$ by sampling a finite number of $\lambda$ values and uses a mini-batch SGD to minimize such an approximated loss. 

At this end, we have shown that the formulation of VRM unifies various effective regularization schemes. The following observations are particularly remarkable. In the original VRM schemes where the vicinal distribution is chosen as $\widehat{p}_{\rm CWBV}$, the extension of the training points $\mathcal{D}$ is only on the $\mathcal{X}$ dimension; namely, each training example $\boldsymbol{x}_i$ is perturbed while preserving its label. In label smoothing, the vicinal distribution $\widehat{p}_{\rm LS}$ extends $\mathcal{D}$ only along the $\mathcal{Y}$ dimension, namely, the location of each $\boldsymbol{x}_i$ is kept but its training label $\boldsymbol{y}_i$ is ``smeared". In MixUp, the extension of the training points $\mathcal{D}$ by the vicinal distribution $\widehat{p}_{\rm MixUp}$ is on both the $\mathcal{X}$ dimension and the $\mathcal{Y}$ dimension. This unique feature of MixUp, which exploits the freedom in the joint space $\mathcal{X}\times\mathcal{Y}$ for moving a training point, is arguably critically attributed to its superior performance relative to other regularization schemes.

We now develop a new regularization scheme that is principled by VRM and exploits the joint space $\mathcal{X}\times\mathcal{Y}$ like MixUp. But unlike MixUp, this scheme is not only designed for regularization but also serve an adversarial training purpose for the ``robust regularization'' objective.

\section{Adversarial Labelling of Perturbed Samples}

Let $\widetilde{p}_\mathcal{X}$ be the marginal of the empirical distribution $\widetilde{p}$ on $\mathcal{X}$. For any $y\in\mathcal{Y}$ and small scalar $\epsilon>0$, we use $\mathcal{B}(\boldsymbol{y};\epsilon)$ to denote the $\text{L}_1$-ball in $\mathcal{Y}$ having radius $\epsilon$ and centred at $\boldsymbol{y}$. We will use $\text{Vol}\left(\mathcal{B}(\boldsymbol{y};\epsilon)\right)$ to denote the volume of $\mathcal{B}(\boldsymbol{y};\epsilon)$. 

For any given $(\boldsymbol{x},\boldsymbol{y})\in\mathcal{X}\times\mathcal{Y}$, consider the following process of constructing a distribution $\gamma(\cdot|\boldsymbol{x},\boldsymbol{y})$ on $\mathcal{X}\times\mathcal{Y}$. First draw random variable $\boldsymbol{X}'$ from the marginal distribution $\widetilde{p}_\mathcal{X}$, and draw random variable $\Lambda$ from Beta distribution $\text{Beta}(\alpha,\beta)$. Then create a pair $(\widehat{\boldsymbol{X}},\widehat{\boldsymbol{Y}})$ of random variables as follows.
\begin{equation}
\widehat{\boldsymbol{X}}:=\mathbb{M}(\boldsymbol{x},\boldsymbol{X}',\Lambda)\qquad\widehat{\boldsymbol{Y}}:=\boldsymbol{y}+\Delta
\end{equation}
where $\Delta$ is uniformly distributed in ball $\mathcal{B}(\boldsymbol{0};\epsilon)$. In our implementation, we set the mode of Beta distribution $\frac{\alpha-1}{\alpha+\beta-2}$ to a value close to 1. Denote the joint distribution of $(\widehat{\boldsymbol{X}},\widehat{\boldsymbol{Y}})$ by $\gamma(\cdot|\boldsymbol{x},\boldsymbol{y})$, and define ALPS vicinal distribution $\widehat{p}_{\rm ALPS}$ as follows: for each $(\boldsymbol{x},\boldsymbol{y})\in\mathcal{X}\times\mathcal{Y}$, 
\begin{equation}
\widehat{p}_{\rm ALPS}(\boldsymbol{x},\boldsymbol{y}):=\int\gamma(\boldsymbol{x},\boldsymbol{y}|\boldsymbol{x}',\boldsymbol{y}')\widetilde{p}(\boldsymbol{x}',\boldsymbol{y}'){\rm d}\boldsymbol{x}'{\rm d}\boldsymbol{y}'
\end{equation}

Under this vicinal distribution, the minimization objective of VRM is 
\begin{align}
&\widehat{\mathcal L}_{\rm ALPS}:=\int\ell\left(f(\boldsymbol{x}),\boldsymbol{y}\right)\widehat{p}_{\rm ALPS}(\boldsymbol{x},\boldsymbol{y}){\rm d}\boldsymbol{x}{\rm d}\boldsymbol{y}\nonumber\\
&=\!\mathbb{E}_\lambda\frac{1}{n^2}\!\sum_{i=1}^n\!\sum_{j=1}^n\!
\frac{1}{\text{Vol}\left(\mathcal{B}(\boldsymbol{y}_i;\!\epsilon)\right)}
\!\!\!\!\!\!
\mathop{\int}_{\boldsymbol{y}\in\mathcal{B}(\boldsymbol{y}_i;\epsilon)}
\!\!\!\!\!\!
\ell\!\left(f\!\left(\mathbb{M}(\boldsymbol{x}_i,\boldsymbol{x}_j;\!\lambda\right)\!,\!\boldsymbol{y}\right)\!{\rm d}\boldsymbol{y}\nonumber\\
&\le\!\mathbb{E}_{\lambda}\frac{1}{n^2}\!\sum_{i=1}^n\!\sum_{j=1}^n\!
\max_{\boldsymbol{y}\in\mathcal{B}(\boldsymbol{y}_i;\epsilon)} 
\!\!\!\ell\!\left(f\!\left(\mathbb{M}(\boldsymbol{x}_i,\boldsymbol{x}_j;\!\lambda\right)\!,\!\boldsymbol{y}\right)\!:=\!\mathcal{J}_{\rm ALPS}\!\! \label{eq:ALPS_final_loss}
\end{align}

Note that in (\ref{eq:ALPS_final_loss}), we turn to an upper bound $\mathcal{J}_{\rm ALPS}$ of the desired minimization objective $\widehat{\mathcal L}_{\rm ALPS}$. This is because it is in general intractable to compute the integral of the loss over the ball $\mathcal{B}(\boldsymbol{y}_i;\epsilon)$. But turning the VRM formulation to minimizing the upper bound $\mathcal{J}_{\rm ALPS}$ may be well justified as follows. 

First, when the radius $\epsilon$ is small, the upper bound is expected to be quite tight. Second, even when the bound is not sufficiently tight, minimizing $\mathcal{J}_{\rm ALPS}$ to a sufficiently small value, say $\eta$, necessarily drives the true objective $\widehat{\mathcal L}_{\rm ALPS}$ to a value no larger than $\eta$. Finally, using the upper bound as the minimization objective renders the training process an ``adversarial'' flavour, and one expects some level of adversarial robustness may be obtained from such minimization. 

To minimize $J_{\rm ALPS}$, one can approximate the expectation over $\lambda$ by its stochastic approximation, and minimizing the double sum can be carried out efficiently using a mini-batch SGD (just as that in MixUp). Concerning the inner maximization, there is an elegant property as shown below.

\begin{theorem}
For an arbitrary training example $(\boldsymbol{x},\boldsymbol{y})\in\mathcal{D}$. Let $\widetilde{\boldsymbol{x}}$ be a perturbed version of $\boldsymbol{x}$, and $k^\ast$ be the least likely label assignment according to the predictive distribution $f(\widetilde{\boldsymbol{x}})$. Define $\rho(\widetilde{\boldsymbol{x}},\boldsymbol{y})\in\mathcal{Y}$ by
$\rho(\widetilde{\boldsymbol{x}},\boldsymbol{y}):=(1-\frac{\epsilon}2)\boldsymbol{y}+\frac{\epsilon}2\vec{\boldsymbol{k}^\ast}$, then
\begin{equation*}
\rho(\widetilde{\boldsymbol{x}},\boldsymbol{y})=\arg\max_{\Vert\widetilde{\boldsymbol{y}}-{\boldsymbol{y}}\Vert_1\le\epsilon}
\ell(f(\widetilde{\boldsymbol{x}}),\widetilde{\boldsymbol{y}})
\end{equation*}
\end{theorem}

\begin{proof}
Obviously, for $\forall\widetilde{\boldsymbol{y}}\in\mathcal{B}(\boldsymbol{y};\epsilon)$, there always $\exists\epsilon'$ satisfy $\epsilon'\le\epsilon$ and $\Vert\widetilde{\boldsymbol{y}}-\boldsymbol{y}\Vert_1=\epsilon'$.
Without loss of the generality, assuming that $\boldsymbol{y}=\vec{\boldsymbol{k}}$, it can then be verified that $\widetilde{\boldsymbol{y}}[k]=1-\frac{\epsilon'}2$.
Then it is possible to show:
\begin{align}
\ell(f(\widetilde{\boldsymbol{x}}),\widetilde{\boldsymbol{y}})&=\sum_{k'\ne k}\widetilde{\boldsymbol{y}}[k']\ell(f(\widetilde{\boldsymbol{x}}),\vec{\boldsymbol{k}}')+(1-\frac{\epsilon'}2)\ell(f(\widetilde{\boldsymbol{x}}),\vec{\boldsymbol{k}})\nonumber\\
&\le\frac{\epsilon'}2\ell(f(\widetilde{\boldsymbol{x}}),\vec{\boldsymbol{k}^\ast})+(1-\frac{\epsilon'}2)\ell(f(\widetilde{\boldsymbol{x}}),\vec{\boldsymbol{k}})
\end{align}

If we continue to slacken the $\text{L}_1$-ball radius $\epsilon'$ until $\epsilon$, $\ell(f(\widetilde{\boldsymbol{x}}),\widetilde{\boldsymbol{y}})$ will reach its supremum.
\begin{align}
&\sup_{\epsilon'\le\epsilon}\left\{\frac{\epsilon'}2\ell(f(\widetilde{\boldsymbol{x}}),\vec{\boldsymbol{k}^\ast})+(1-\frac{\epsilon'}2)\ell(f(\widetilde{\boldsymbol{x}}),\vec{\boldsymbol{k}})\right\}\nonumber\\
=&\frac{\epsilon}2\ell(f(\widetilde{\boldsymbol{x}}),\vec{\boldsymbol{k}^\ast})+(1-\frac{\epsilon}2)\ell(f(\widetilde{\boldsymbol{x}}),\vec{\boldsymbol{k}})\nonumber\\
=&\ell(f(\widetilde{\boldsymbol{x}}),(1-\frac{\epsilon}2)\vec{\boldsymbol{k}}+\frac{\epsilon}2\vec{\boldsymbol{k}^\ast})\nonumber\\
=&\ell(f(\widetilde{\boldsymbol{x}}),\rho(\widetilde{\boldsymbol{x}},\boldsymbol{y}))
\end{align}
this suggests that $\widetilde{\boldsymbol{y}}=\rho(\widetilde{\boldsymbol{x}},\boldsymbol{y})$ is the maximizer of $\ell(f(\widetilde{\boldsymbol{x}}),\widetilde{\boldsymbol{y}})$ given the perturbed example $\widetilde{\boldsymbol{x}}$.
\end{proof}

The theorem suggests that the inner maximization in (\ref{eq:ALPS_final_loss}) has the ``adversarial label'' $\rho(\widetilde{\boldsymbol{x}},\boldsymbol{y})$ as its closed-form solution. Note that $\rho(\widetilde{\boldsymbol{x}},\boldsymbol{y})$ is merely the label distribution obtained by shifting probability $\epsilon/2$ from the ground-truth label to the least likely label. The minimizing $\mathcal{J}_{\rm ALPS}$ reduces to
\begin{align}
&\min\mathbb{E}_{\lambda}
\frac{1}{n^2}\!\sum_{i=1}^{n}\!\sum_{j=1}^n\!\ell\!\left(f\!\left(\mathbb{M}(\boldsymbol{x}_i,\boldsymbol{x}_j;\lambda)\right)\!,\rho(\mathbb{M}(\boldsymbol{x}_i,\boldsymbol{x}_j;\lambda),\boldsymbol{y}_i)\right)\nonumber\\
\approx&
\min\frac{1}{|\mathcal{S}|}\sum_{(\boldsymbol{x}_i,\boldsymbol{x}_j)\in\mathcal{S}}\ell\!\left(f\!\left(\widetilde{\boldsymbol{x}}_{ij}\right),\rho(\widetilde{\boldsymbol{x}}_{ij},\boldsymbol{y}_i)\right)
\end{align}
where the above stochastic approximation uses one batch $\mathcal{S}$ of training example pairs, and for example pair $(\boldsymbol{x}_i,\boldsymbol{x}_j)$, we have used $\widetilde{\boldsymbol{x}}_{ij}$ to denote the mix $\mathbb{M}(\boldsymbol{x}_i,\boldsymbol{x}_j;\lambda_{ij})$ with $\lambda_{ij}$ drawn from $\text{Beta}(\alpha,\beta)$. The gradient signal of this approximation can then be used to update the model parameters as usual.

\begin{algorithm}[htb] 
\caption{ALPS Regularized Training} 
\label{alg:Framwork} 
\begin{algorithmic}[1] %
\REQUIRE The training set $\mathcal{D}=\{(\boldsymbol{x}_i,\boldsymbol{y}_i)\}_{i=1}^n$, the neural network $f(\cdot)$, the number of epochs $E$, the hyper-parameters of beta distribution $\alpha,\beta$, and the radius $\epsilon$ of permitted perturbation over label space $\mathcal{Y}$
\FOR{$e=1$ to $E$}
\FOR{$i=1$ to $n$}
\STATE Sample $\boldsymbol{x}'\sim\mathcal{D}$ uniformly, sample $\lambda\sim\text{Beta}(\alpha,\beta)$;\\
\STATE Let $\widetilde{\boldsymbol{x}}=\mathbb{M}(\boldsymbol{x}_i,\boldsymbol{x}';\lambda)$;\\
\STATE Get the least likely label $k^\ast$ predicted from $f(\widetilde{\boldsymbol{x}})$, and assign $\rho=(1-\frac{\epsilon}2){\boldsymbol{y}}_i+\frac{\epsilon}2\vec{\boldsymbol{k}^\ast}$;\\
\STATE Use the vicinal sample $(\widetilde{\boldsymbol{x}},\rho)$ to compute the loss gradient and update the parameters;
\ENDFOR
\ENDFOR
\end{algorithmic}
\end{algorithm}

\paragraph{Remarks on ALPS's robust regularization effects}
Note that an overfitted model tends to carve its class boundaries too close to the training examples, and its corresponding loss landscape tends to change sharply near the training points. This is mitigated in ALPS by training the model with synthetic data points that have moved away from the data manifold in the joint input-label space. This aspect is similar to MixUp. On the other hand, the movement in the label space is designed in an adversarial manner, namely, towards the direction of increasing the loss. This aspect is similar to adversarial training. Thus ALPS is expected to work both as a data-dependent regularization scheme and as an adversarial training scheme.

\section{Experiments}

Here we experimentally compare the regularization performance of ALPS with other popular regularization schemes. Adversarial robustness of the trained models is also evaluated and compared. We select several popular regularization schemes, including Dropout, Label smoothing, MixUp, and vanilla ERM as baselines to compare with ALPS. Weight decay is applied to all schemes.

\begin{table}[t]
\centering
\resizebox{0.8\columnwidth}{!}{
\begin{tabular}{lrr}
\toprule
Method & \multicolumn{1}{c}{PreActResNet18} & \multicolumn{1}{c}{VGG16} \\
\midrule
ERM & 4.06$\pm$0.063 & 4.21$\pm$0.060 \\
Dropout & 3.80$\pm$0.065 & 3.85$\pm$0.049 \\
Label Smoothing & 3.31$\pm$0.087 & 3.70$\pm$0.070 \\
MixUp & 3.35$\pm$0.044 & 3.60$\pm$0.057 \\
ALPS & {\bf 3.05$\pm$0.045} & {\bf 3.48$\pm$0.051} \\
\bottomrule
\end{tabular}
}%
\caption{Test error (\%) on SVHN.}
\label{tab:svhn}
\end{table}

\begin{table}[t]
\centering
\resizebox{0.8\columnwidth}{!}{
\begin{tabular}{lrr}
\toprule
Method & \multicolumn{1}{c}{PreActResNet18} & \multicolumn{1}{c}{VGG16} \\
\midrule
ERM & 6.13$\pm$0.212 & 6.75$\pm$0.193 \\
Dropout & 5.91$\pm$0.148 & 6.42$\pm$0.147 \\
Label Smoothing & 5.66$\pm$0.135 & 6.63$\pm$0.157 \\
MixUp & 4.22$\pm$0.117 & 5.46$\pm$0.112 \\
ALPS & {\bf 4.15$\pm$0.129} & {\bf 5.19$\pm$0.144} \\
\bottomrule
\end{tabular}
}
\caption{Test error (\%) on CIFAR-10.}
\label{tab:cifar10}
\end{table}

\begin{table}[t]
\centering
\resizebox{0.82\columnwidth}{!}{
\begin{tabular}{lrr}
\toprule
Method & \multicolumn{1}{c}{PreActResNet18} & \multicolumn{1}{c}{VGG16} \\
\midrule
ERM & 25.97$\pm$0.303 & 28.87$\pm$0.297 \\
Dropout & 25.46$\pm$0.351 & 27.84$\pm$0.337 \\
Label Smoothing & 24.85$\pm$0.256 & 28.69$\pm$0.179 \\
MixUp & 22.07$\pm$0.210 & 26.50$\pm$0.254 \\
ALPS & {\bf 21.75$\pm$0.208} & {\bf 26.18$\pm$0.233} \\
\bottomrule
\end{tabular}
}
\caption{Test error (\%) on CIFAR-100.}
\label{tab:cifar100}
\end{table}

\begin{table}[t]
\centering
\resizebox{0.98\columnwidth}{!}{
\begin{tabular}{lrr}
\toprule
Method & \multicolumn{1}{c}{PreActResNet18} & \multicolumn{1}{c}{VGG16} \\
\midrule
ERM & 39.05$\pm$0.151\ /\ 18.25 & 41.74$\pm$0.211\ /\ 19.79 \\
Dropout & 37.52$\pm$0.149\ /\ 16.81 & 38.93$\pm$0.182\ /\ 19.07 \\
Label Smoothing & 38.41$\pm$0.136\ /\ 19.46 & 39.09$\pm$0.174\ /\ 19.75 \\
MixUp & 37.05$\pm$0.119\ /\ 16.24 & 38.29$\pm$0.140\ /\ 18.35 \\
ALPS & {\bf 36.23$\pm$0.121}\ /\ {\bf 15.42} & {\bf 38.17$\pm$0.135}\ /\ {\bf 18.04} \\
\bottomrule
\end{tabular}
}
\caption{Test error (\%) on Tiny-ImageNet. The first number and the second number are the top-1 and top-5 error rates, respectively.}
\label{tab:tinyimagenet}
\end{table}

\subsection{Experimental Setup}

Four publicly available benchmark datasets are used. The SVHN dataset has 10 classes for the digit numbers, and 73257 training digits, 26032 test digits. The CIFAR-10 and CIFAR-100 dataset have the same image set but different label strategy. CIFAR-10 has 10 classes containing 6000 images each, and CIFAR-100 has 100 classes containing 600 images each. The two datasets are split with 5:1 for training and testing per class. The Tiny-ImageNet dataset has 200 classes. Each class has 500 training images, 50 validation images, and 50 test images. Since the test set does not provide true labels, we take the validation set for testing.

Besides the vanilla empirical risk minimization (ERM), we select three regularization methods to compare with ALPS, {\em i.e.}, Dropout, label smoothing, and MixUp.
PreActResNet18 \cite{he2016identity} and VGG16 \cite{simonyan2014very} are employed as classification model architecture. In training processes, we adopt the geometrical transformation augmentation methods, random crops and horizontal flips. Normalization is performed on both training and test set with the mean and standard derivation computed on training set.

We follow the training procedure of \cite{zhang2018mixup}, where SGD with momentum optimizer is used with a step-wise learning rate decay. Weight decay factor is set to $10^{-4}$. For Dropout, we randomly deactivate 50\% of the neurons in the penultimate layer at each iteration. The label smoothing ratio is set to $0.1$. The Beta hyper-parameter of MixUp $\alpha$ is set to $1$. For ALPS, the mode of asymmetric Beta distribution is chosen in the range $(0.75,1)$, and the $\text{L}_1$ ball constraint $\epsilon$ in adversarial labelling is restricted to $(0,0.5)$ empirically.

\subsection{Regularization Performance}

On SVHN, CIFAR-10 and CIFAR-100 datasets, top-1 test error is used to measure the generalization ability of PreActResNet18 and VGG16 networks trained with the compared regularization schemes. The corresponding results are reported in Table~\ref{tab:svhn}, Table~\ref{tab:cifar10}, and Table~\ref{tab:cifar100} respectively. ALPS outperforms other baselines generally. On average, ALPS trained the two models' top-1 error are improved by about $1\%$ on SVHN, $2\%$ on CIFAR-10, and $3.5\%$ on CIFAR-100 respectively, in contrast to ERM baseline. The performance of MixUp is close with and slightly worse than that of ALPS. On the Tiny-ImageNet, top-1 and top-5 errors are used to measure regularization performance. The results are summarized in Table~\ref{tab:tinyimagenet}. We also find that ALPS achieves the best performance compared with all other baseline methods. Especially for PreActResNet18 model, the performance of ALPS is higher than the competitive MixUp method, and gains an increase of $0.8\%$ in both top-1 and top-5 error compared with it. These results confirm that ALPS is a state-of-art regularization scheme.

\begin{table*}[t]
\centering
\resizebox{0.98\textwidth}{!}{
\begin{tabular}{lrrrr|rrrr}
\toprule
& \multicolumn{4}{c}{white-box} & \multicolumn{4}{c}{black-box} \\
Method & \multicolumn{1}{c}{SVHN} & \multicolumn{1}{c}{CIFAR-10} & \multicolumn{1}{c}{CIFAR-100} & \multicolumn{1}{c}{Tiny-ImageNet} & \multicolumn{1}{c}{SVHN} & \multicolumn{1}{c}{CIFAR-10} & \multicolumn{1}{c}{CIFAR-100} & \multicolumn{1}{c}{Tiny-ImageNet} \\
\midrule
ERM & 34.70$\pm$0.569 & 39.02$\pm$1.908 & 68.67$\pm$0.699 & 87.96$\pm$0.735 & 10.37$\pm$0.341 & 12.41$\pm$0.916 & 37.99$\pm$0.750 & 50.34$\pm$0.683 \\
Dropout & 34.23$\pm$0.591 & 38.41$\pm$0.844 & 65.70$\pm$0.549 & \underline{86.28$\pm$0.614} & 9.94$\pm$0.472 & 12.20$\pm$0.900 & 35.94$\pm$0.536 & 45.70$\pm$0.702 \\
Label Smoothing & \underline{19.28$\pm$1.674} & \underline{27.38$\pm$0.427} & \underline{63.26$\pm$0.410} & 87.31$\pm$0.436 & 9.11$\pm$1.772 & 10.86$\pm$0.507 & 34.29$\pm$0.364 & 48.80$\pm$0.499 \\
MixUp & 27.60$\pm$0.637 & 31.96$\pm$0.384 & 68.33$\pm$0.498 & 87.59$\pm$0.417 & \underline{7.92$\pm$0.212} & \underline{9.79$\pm$0.340} & \underline{33.74$\pm$0.477} & \underline{44.91$\pm$0.314} \\
ALPS & {\bf 15.85$\pm$0.965} & {\bf 24.00$\pm$0.572} & {\bf 56.93$\pm$0.515} & {\bf 83.78$\pm$0.684} & {\bf 6.46$\pm$0.415} & {\bf 7.91$\pm$0.576} & {\bf 32.84$\pm$0.237} & {\bf 43.80$\pm$0.410} \\
\bottomrule
\end{tabular}
}
\caption{Test error (\%) on white-box and black-box FGSM attacks.}
\label{tab:fgsm}
\end{table*}

\begin{table*}[t]
\centering
\resizebox{0.98\textwidth}{!}{
\begin{tabular}{lrrrr|rrrr}
\toprule
& \multicolumn{4}{c}{white-box} & \multicolumn{4}{c}{black-box} \\
Method & \multicolumn{1}{c}{SVHN} & \multicolumn{1}{c}{CIFAR-10} & \multicolumn{1}{c}{CIFAR-100} & \multicolumn{1}{c}{Tiny-ImageNet} & \multicolumn{1}{c}{SVHN} & \multicolumn{1}{c}{CIFAR-10} & \multicolumn{1}{c}{CIFAR-100} & \multicolumn{1}{c}{Tiny-ImageNet} \\
\midrule
ERM & 54.34$\pm$1.085 & 58.33$\pm$2.296 & 83.03$\pm$0.770 & 94.34$\pm$0.637 & 14.14$\pm$0.993 & 15.98$\pm$1.784 & 41.29$\pm$0.892 & 48.98$\pm$0.571 \\
Dropout & 55.46$\pm$1.346 & 59.19$\pm$1.088 & \underline{79.15$\pm$1.081} & \underline{92.48$\pm$0.671} & 12.67$\pm$1.122 & 16.17$\pm$1.003 & 39.82$\pm$0.909 & 45.35$\pm$0.723 \\
Label Smoothing & \underline{34.20$\pm$2.005} & {\bf 42.78$\pm$0.443} & 79.65$\pm$0.700 & 93.47$\pm$0.813 & 11.82$\pm$2.233 & 13.80$\pm$0.853 & 40.24$\pm$0.936 & 47.82$\pm$0.607 \\
MixUp & 84.80$\pm$2.129 & 71.02$\pm$1.018 & 88.80$\pm$0.596 & 96.48$\pm$1.074 & \underline{11.27$\pm$1.682} & \underline{12.78$\pm$1.122} & \underline{35.85$\pm$0.417} & \underline{44.10$\pm$0.722} \\
ALPS & {\bf 28.53$\pm$0.960} & \underline{44.55$\pm$0.650} & {\bf 68.34$\pm$0.488} & {\bf 91.34$\pm$0.593} & {\bf 7.29$\pm$0.703} & {\bf 8.55$\pm$0.584} & {\bf 33.79$\pm$0.470} & {\bf 42.79$\pm$0.633} \\
\bottomrule
\end{tabular}
}
\caption{Test error (\%) on white-box and black-box PGD attacks.}
\label{tab:pgd}
\end{table*}

\subsection{Robustness to Adversarial Attacks}

Training a model with out-of-manifold samples as well as label adversaries, ALPS is expected to substantially improved the adversarial robustness of the trained model. To validate this, we consider the off-the-shelf Fast Gradient Sign Method (FGSM) \cite{goodfellow2014explaining} and Projected Gradient Decent method (PGD) \cite{madry2018towards}, to construct adversarial examples. Then, we train the PreActResNet18 model with ALPS and all the baseline methods, and evaluate the robustness against the two kinds of attacks in both white-box and black-box settings. 

For the setting of white-box adversarial attacks, we use the trained models themselves to generate attack samples. The top-1 test error is also adopted as the evaluation metric. The observed best and second-best performances are marked with bold font and underline respectively. Table~\ref{tab:fgsm} (left) lists the FGSM attack experiment results with perturbation radius $\epsilon=4$ for each pixel. It can be observed that ALPS generally outperforms other regularization baselines, with the gaps of $3.43\%$ on the SVHN dataset, $3.38\%$ on the CIFAR-10 dataset, $6.33\%$ on the CIFAR-100 dataset, and $2.50\%$ on the Tiny-ImageNet dataset compared with the correspondingly second-best schemes. In Table~\ref{tab:pgd} (left), ALPS also manifests superior performance in defending against the 10-step PGD white-box attacks with the steps size set to 1 and perturbation radius $\epsilon=4$. ALPS is seen to dominate in all other experiment settings except in CIFAR-10 where ALPS appears weaker than Label Smoothing baseline with a gap $1.77\%$. It is interesting to note that Dropout and label smoothing appear to perform well with respect to defending both FGSM and PGD white-box adversarial attacks. On the other hand, MixUp appears quite weak against PGD attacks, and this weakness is seen nicely cured in ALPS, arguably attributed to the adversarial labelling in ALPS.

For the black-box adversarial attack setting, we firstly train a model with ERM to generate a set of adversarial examples. Then we measure the adversarial robustness of the other methods on the set of adversarial examples. Table~\ref{tab:fgsm} (right) and Table~\ref{tab:pgd} (right) summaries the evaluation results for FGSM and PGD respectively. Similarly to white-box attacks, ALPS outperforms all other baselines. It is worth mentioning that the models trained with MixUp also achieves significant improvement on all three datasets. In addition, the Dropout and label smoothing methods are relatively not much effective to black-box attacks in contrast to the competitive performances in the white-box attack experiments.

\paragraph{Sanity Check for Adversarial Attacks}
The issue known as ``gradient masking'' \cite{athalye2018obfuscated} may lead to false sense of adversarial robustness owing to the inferior quality of the gradients, {\em e.g.}, shattered gradients, vanishing gradients. To validate this, we further conduct unbounded PGD sanity check as \cite{verma2019manifold}. We take our trained models with MixUp and ALPS on the four datasets, and run PGD for 200 iterations with a step size $0.01$. The unbounded PGD attack reduced the ALPS’s accuracy to $0.22\%$ on SVHN dataset and reduce the accuracy of other seven cases to $0\%$. This suggests that there are no obvious obfuscated gradients in our robustness measurement. 

\section{Conclusion}

In this paper, we suggest VRM as a general principle for regularization. Under this perspective, MixUp appears to have a combined advantage of label smoothing and the traditional VRM schemes. Using an approach similar to MixUp, we construct another vicinal joint distribution of synthetic samples and their adversarial labels. Through minimizing the upper bound of the VRM loss, this formulation gives rise to a novel regularization scheme, ALPS. We show that ALPS not only demonstrates a state-of-the-art regularization performance, it also provides the trained model with strong adversarial robustness. This paper appears to be the first work that develops a strong regularization scheme while also explicitly aiming at adversarial robustness. We hope that further research will be inspired towards this ``robust regularization'' objective.

\section*{Acknowledgements}

This work is supported partly by the National Natural Science Foundation of China (No. 61772059), by the Fundamental Research Funds for the Central Universities, by the Beijing S\&T Committee (No. Z191100008619007), by National Key Research and Development Program (2018YFB1306000), by the Beijing Advanced Innovation Center for Big Data and Brain Computing and by State Key Laboratory of Software Development Environment (No. SKLSDE-2020ZX-14).

\bibliographystyle{named}
\bibliography{main}

\end{document}